\documentclass[runningheads]{llncs}
\usepackage{hyperref}
\usepackage{soul}

\newcommand{\myparagraph}[1]{\textbf{#1}\hspace{0.3em}}

\usepackage{tikz, pgfplots}
\usetikzlibrary{positioning}
\usetikzlibrary{external}

\usepackage{pifont} 
\definecolor{mygreen}{RGB}{35,110,25}
\newcommand{\mycheckmark}{\ding{51}}
\newcommand{\myxmark}{\ding{55}}

\usepackage{cite}
\usepackage{xspace}
\usepackage{graphicx} 
\usepackage{subfigure}
\usepackage{subcaption}
\usepackage{amsmath}
\usepackage{bm}
\usepackage{empheq}
\usepackage{booktabs}
\usepackage{tikz}
\usepackage{multicol}
\usepackage{multirow}
\usepackage{placeins}
\usepackage{wrapfig}
\usepackage{tabularx}
\usepackage{rotating}
\usepackage{amssymb}
\usepackage{mathtools}
\usepackage{mathrsfs}

\usepackage[T1]{fontenc}
\usepackage{graphicx}
\begin{document}
\title{HeNCler: Node Clustering in Heterophilous Graphs via Learned Asymmetric Similarity}
\titlerunning{Node Clustering in Heterophilous Graphs via Learned Asymmetric Similarity}
\author{
Sonny Achten\inst{1}
\and
Zander Op de Beeck\inst{1}
\and
Francesco Tonin\inst{2}
\and
Volkan Cevher\inst{2}
\and
Johan A. K. Suykens\inst{1}\thanks{This is the submitted version of the manuscript accepted at the International Conference on Artificial Neural Networks (ICANN) 2025, Special Session on Neural Networks for Graphs and Beyond. The final Version of Record will appear in Springer LNCS.}
}
\authorrunning{S. Achten et al.}
\institute{
  ESAT-STADIUS, KU Leuven, Leuven, Belgium
  \email{\{sonny.achten,zander.opdebeeck,johan.suykens\}@kuleuven.be}
  \and
  LIONS, EPFL, Lausanne, Switzerland \\
  \email{\{francesco.tonin,volkan.cevher\}@epfl.ch}
  }
\maketitle              
\begin{abstract} 
Clustering nodes in heterophilous graphs is challenging as traditional methods assume that effective clustering is characterized by high intra-cluster and low inter-cluster connectivity. To address this, we introduce HeNCler—a novel approach for \underline{\textbf{He}}terophilous \underline{\textbf{N}}ode \underline{\textbf{Cl}}ust\underline{\textbf{er}}-ing. 
HeNCler  \emph{learns} a similarity graph by optimizing a clustering-specific objective based on weighted kernel singular value decomposition.
Our approach enables spectral clustering on an \emph{asymmetric} similarity graph, providing flexibility for both directed and undirected graphs. By solving the primal problem directly, our method overcomes the computational difficulties of traditional adjacency partitioning-based approaches. Experimental results show that HeNCler significantly improves node clustering performance in heterophilous graph settings, highlighting the advantage of its asymmetric graph-learning framework. 

\keywords{Heterophily, Clustering, Kernel SVD}
\end{abstract}
\section{Introduction}
{Graph} neural networks (GNNs) have substantially advanced machine learning applications to graph-structured data by effectively propagating node attributes end-to-end. Typically, GNNs rely on the assumption of homophily, where nodes with similar labels are more likely to be connected \cite{zheng2024graph}.
The homophily assumption holds in contexts such as social networks and citation graphs. It underlies the effectiveness of message-passing GNNs like GCN \cite{Kipf:2017tc} in tasks such as node classification and graph prediction. In homophilous graphs, local feature aggregation reinforces meaningful representations through neighborhood smoothing.

In heterophilous graphs, such as web or transaction networks, connected nodes often differ in labels, making local aggregation less effective. This requires capturing long-range dependencies by learning edge importance \cite{Dwivedi_graph_trand, bi2024}, which is more feasible in supervised tasks due to label guidance.

Our work specifically targets attributed node clustering in a fully unsupervised setting, where neither labels nor external guidance are available---necessitating the use of unsupervised or self-supervised learning strategies.
For instance, auto-encoder type models \cite{Pan_2020} are primarily focused on node representation learning rather than clustering, making them less suited for directly improving cluster-ability. 
Various self-supervised, contrastive learning techniques \cite{Hassani2020ContrastiveMR, GraphCL_You_2020} enhance node representation learning in homophilous settings only.
At the same time, other self-supervised methods have been developed to handle heterophilous graphs \cite{chen2022towards, DSSL_Xiao_2022}. 
However, these methods are designed for the general node representation learning task and lack a clustering objective.

Existing node clustering methods \cite{bianchi_spectral_2020, tsitsulin_graph_2023, devvrit_s3gc_2022} often rely on proximity-based objectives or partitioning of the adjacency matrix, assuming that effective clustering aligns with high intra-cluster and low inter-cluster similarity---a premise often invalid in heterophilous graphs. Moreover, these methods are typically limited to undirected graphs, disregarding valuable asymmetric information.

This paper introduces HeNCler, a novel approach for node clustering in heterophilous graphs, illustrated in Figure \ref{fig:graph_learning}.
Existing work overlooks the asymmetric relationships in heterophilous graphs.
HeNCler addresses this by using weighted kernel singular value decomposition (wKSVD) to induce a learned asymmetric similarity graph for both directed and undirected graphs. 
The dual problem of wKSVD aligns with asymmetric kernel spectral clustering, enabling the interpretation of similarities without homophily. 
By solving the primal problem directly, HeNCler overcomes computational difficulties and shows superior performance in node clustering tasks within heterophilous graphs.

\myparagraph{Our contributions} in this work can be summarized as follows:
\begin{itemize}
\item We introduce HeNCler, a kernel spectral biclustering framework designed to \textit{learn} an induced \textit{asymmetric} similarity graph suited for node clustering of heterophilous graphs, applicable to both directed and undirected graphs.
\item We develop a primal-dual framework for a generic weighted kernel singular value decomposition (wKSVD) model.
\item We show that the dual wKSVD formulation allows for biclustering of bipartite/asymmetric graphs, while we employ a computationally feasible implementation in the primal wKSVD formulation.
\item We further generalize our approach with trainable feature mappings, using node and edge decoders, such that the similarity matrix to cluster is learned.
\item We train HeNCler in the primal setting and demonstrate its superior performance on the node clustering task for heterophilous attributed graphs. 
\item Our implementation and supplementary materials are available on GitHub: \url{https://github.com/sonnyachten/HeNCler/}.
\end{itemize}
\begin{figure}
    \centering
    \resizebox{0.85\linewidth}{!}{
    \begin{tikzpicture}[scale=0.6]
        \node[draw, circle, fill=blue!50, minimum size=0.37cm, inner sep=0pt] (a) at (0, 0.5) {\scriptsize{$u$}};
        \node[draw, circle, fill=blue!50, minimum size=0.37cm, inner sep=0pt] (d) at (1, -0.5) {\scriptsize{$v$}};
        \node[draw, circle, fill=yellow!50, minimum size=0.37cm] (b) at (1, 0.5) {};
        \node[draw, circle, fill=yellow!50, minimum size=0.37cm] (c) at (0, -0.5) {};
        \node[draw, circle, fill=yellow!50, minimum size=0.37cm] (e) at (2, -0.5) {};
        \draw[->] (a) -- (b);
        \draw[->] (c) -- (a);
        \draw[->] (d) -- (b);
        \draw[->] (d) -- (c);
        \draw[->] (e) -- (d);

        \draw[->, line width = 0.3mm] (2.8, 0) -- (3.6, 0);
        \draw[->, line width = 0.3mm] (7.3, 0) -- (8.1, 0); 

        \node[draw, circle, fill=blue!50, minimum size=0.37cm, inner sep=0pt] (f) at (4.5, 2) {\scriptsize{$u$}};
        \node[draw, circle, fill=blue!50, minimum size=0.37cm, inner sep=0pt] (g) at (4.5, 1) {\scriptsize{$v$}};
        \node[draw, circle, fill=yellow!50, minimum size=0.37cm] (h) at (4.5, 0) {};
        \node[draw, circle, fill=yellow!50, minimum size=0.37cm] (i) at (4.5, -1) {};
        \node[draw, circle, fill=yellow!50, minimum size=0.37cm] (j) at (4.5, -2) {};

        \node[draw, circle, fill=blue!50, minimum size=0.37cm, inner sep=0pt] (k) at (6.5, 2) {\scriptsize{$u$}};
        \node[draw, circle, fill=blue!50, minimum size=0.37cm, inner sep=0pt] (l) at (6.5, 1) {\scriptsize{$v$}};
        \node[draw, circle, fill=yellow!50, minimum size=0.37cm] (m) at (6.5, 0) {};
        \node[draw, circle, fill=yellow!50, minimum size=0.37cm] (n) at (6.5, -1) {};
        \node[draw, circle, fill=yellow!50, minimum size=0.37cm] (o) at (6.5, -2) {};

        \draw (f) -- (k);
        \draw (g) -- (l);
        \draw[red, line width = 0.5mm] (f) -- (l);
        \draw[blue, line width = 0.5mm] (g) -- (k);

        \draw (h) -- (m);
        \draw (h) -- (n);
        \draw (h) -- (o);
        \draw (i) -- (m);
        \draw (i) -- (n);
        \draw (i) -- (o);        
        \draw (j) -- (m);
        \draw (j) -- (n);
        \draw (j) -- (o);

        \node at (5.5, -2.6) {\scriptsize{$\mathcal{S}=(\bm{\Phi},\bm{\Psi},\bm{S})$}};
        \node at (4.5, 2.7) {\scriptsize{$\phi(\bm{x})$}};
        \node at (6.5, 2.7) {\scriptsize{$\psi(\bm{x})$}};
        \node at (3.35, 1.52) {\tiny{\textcolor{red}{$\text{sim}(u, v)$}$\neq$\textcolor{blue}{$\text{sim}(v, u)$}}};

        \node[draw, circle, fill=blue!50, minimum size=0.37cm] (p) at (9, 2) {};
        \node[draw, circle, fill=blue!50, minimum size=0.37cm] (q) at (9, 1) {};
        \node[draw, circle, fill=yellow!50, minimum size=0.37cm] (r) at (9, 0) {};
        \node[draw, circle, fill=yellow!50, minimum size=0.37cm] (s) at (9, -1) {};
        \node[draw, circle, fill=yellow!50, minimum size=0.37cm] (t) at (9, -2) {};

        \node[draw, circle, fill=blue!50, minimum size=0.37cm] (u) at (11, 2) {};
        \node[draw, circle, fill=blue!50, minimum size=0.37cm] (v) at (11, 1) {};
        \node[draw, circle, fill=yellow!50, minimum size=0.37cm] (w) at (11, 0) {};
        \node[draw, circle, fill=yellow!50, minimum size=0.37cm] (x) at (11, -1) {};
        \node[draw, circle, fill=yellow!50, minimum size=0.37cm] (y) at (11, -2) {};

        \draw (p) -- (u);
        \draw (p) -- (v);
        \draw (q) -- (u);
        \draw (q) -- (v);

        \draw (r) -- (w);
        \draw (r) -- (x);
        \draw (r) -- (y);
        \draw (s) -- (w);
        \draw (s) -- (x);
        \draw (s) -- (y);        
        \draw (t) -- (w);
        \draw (t) -- (x);
        \draw (t) -- (y);

        \draw[dashed] (9, 1.5) ellipse (0.55 and 0.95);
        \draw[dashed] (9, -1) ellipse (0.6 and 1.45);
        \draw[dashed] (11, 1.5) ellipse (0.55 and 0.95);
        \draw[dashed] (11, -1) ellipse (0.6 and 1.45);
    \end{tikzpicture}
    }
    \caption{\textbf{HeNCler Overview}. Starting from a heterophilous (directed) graph, where similar nodes are far apart (left), HeNCler learns two sets of node representations, $\{\phi(\bm{x}_v)\}_{v \in \mathcal{V}}$ and $\{\psi(\bm{x}_v)\}_{v \in \mathcal{V}}$, forming a bipartite graph $\mathcal{S}$ (middle), where the similarity between nodes is defined as $ S_{uv} = \text{sim}(u,v)=\phi(\bm{x}_u)^\top\psi(\bm{x}_v)$. 
    The clustering objective brings related nodes closer in the learned graph, and clusters are then identified using spectral biclustering through wKSVD (right).}
    \label{fig:graph_learning}
\end{figure}
\section{Preliminaries and related work}
We use lowercase symbols (e.g., $x$) for scalars, lowercase bold (e.g., $\bm{x}$) for vectors and uppercase bold (e.g., $\bm{X}$) for matrices. A single entry of a matrix is represented by $X_{ij}$. $\phi(\cdot)$ denotes a mapping and $\bm{\phi}_v=\phi(\bm{x}_v)$ represents the mapping of node $v$ in the induced feature space.
We represent a graph $\mathcal{G}$ by its vertices (i.e., nodes) $\mathcal{V}$ and edges $\mathcal{E}$, $\mathcal{G}=(\mathcal{V},\mathcal{E})$, or by its node feature matrix and adjacency matrix $\mathcal{G}=(\bm{X},\bm{A})$. For a bipartite graph, we have $\mathcal{G}=(\mathcal{I},\mathcal{J},\mathcal{E})$ or  $\mathcal{G}=(\bm{X}_\mathcal{I},\bm{X}_\mathcal{J},\bm{S})$ where $S_{ij}$ is the edge weight between nodes $i\in\mathcal{I}$ and $j\in\mathcal{J}$. Note that $\bm{S}$ is generally asymmetric and rectangular, and that the adjacency matrix of the bipartite graph is given by $\bm{A} =$ \begin{small}$\left[\begin{smallmatrix}
			\bm{0} & \bm{S} \\ \bm{S}^{\top} & \bm{0} \end{smallmatrix}\right]$ \end{small}. 

\myparagraph{Kernel singular value decomposition}(KSVD) \cite{suykens_svd_2016} 
allows for non-linear extensions of the SVD problem and can be applied on data structures such as row and column features, directed graphs, and/or can exploit asymmetric similarity information such as conditional probabilities \cite{He2023}. Interestingly, KSVD often outperforms the similar, though symmetric, kernel principal component analysis model on tasks where asymmetry is not immediately apparent \cite{He2023, tao_nonlinear_2023}. 

\myparagraph{Spectral clustering}generalizations have been proposed in many settings. Spectral graph biclustering \cite{dhillon_co-clustering_2001} formulates the spectral clustering problem of a bipartite graph $\mathcal{G}=(\mathcal{I},\mathcal{J},\bm{S})$ and shows the equivalence with the SVD of the normalized matrix $\bm{S}_n = \bm{D}_1^{-1/2} \bm{S} \bm{D}_2^{-1/2}$, where $D_{1,ii}=\sum_j S_{ij}$ and  $D_{2,jj}=\sum_i S_{ij}$. Cluster assignments for nodes $\mathcal{I}$ and nodes $\mathcal{J}$ can be inferred from the left and right singular vectors respectively. Further, kernel spectral clustering (KSC) \cite{KSC} proposes a weighted kernel principal component analysis in which the dual formulation corresponds to the random walks interpretation of the spectral clustering problem. KSC and the aforementioned spectral biclustering formulation lack asymmetry and a primal formulation respectively, which are limitations that our model addresses.   

\myparagraph{Restricted kernel machines}(RKM) \cite{suykens_deep_2017} possess primal and dual model formulations, based on the concept of conjugate feature duality. It is an energy-based framework for (deep) kernel machines.
RKMs encompasses many model classes, including classification, regression, kernel principal component analysis, and KSVD, and allows deep kernel learning and deep kernel learning on graphs \cite{achten_gckm_2024}. One possibility of representing feature maps in RKMs is by means of deep neural networks, e.g., for unsupervised representation learning \cite{strkm}. 

\myparagraph{Homophilous node clustering} methods such as MinCutPool \cite{bianchi_spectral_2020} and DMoN \cite{tsitsulin_graph_2023} introduce unsupervised loss functions in a graph neural network framework. MinCutPool employs a relaxed version of the minimal cut objective applied to the adjacency matrix, while DMoN optimizes the modularity score of clustering assignments with respect to the graph structure. Both methods rely on partitioning the adjacency matrix and inherently assume strong homophily. Moreover, due to their theoretical foundations, these losses are restricted to undirected graphs.
Beyond adjacency partitioning, self-supervised and contrastive learning techniques have also been proposed \cite{GraphCL_You_2020, Hassani2020ContrastiveMR, devvrit_s3gc_2022}. These approaches typically use graph proximity as a supervision signal, again assuming a degree of homophily. For instance, S$^3$GC \cite{devvrit_s3gc_2022} leverages random walk co-occurrences to infer proximity-based similarities.

\myparagraph{Heterophilous node clustering} methods predominantly rely on contrast\-ive or self-supervised techniques that are less dependent on local proximity. 
HoLe \cite{gu2023homophily} addresses unsupervised graph structure learning by iteratively rewiring the graph using intermediate clustering results. This process increases effective homophily, thereby improving clustering performance on moderately heterophilous graphs.
SparseGAD \cite{SparseGAD_Gong_2023} sparsifies graph structures to effectively reduce noise from irrelevant edges and enhance the detection of closely related nodes. In contrast, methods like ours aim to learn an entirely new similarity graph, which is particularly valuable for strongly heterophilous settings.
Other approaches, such as HGRL \cite{chen2022towards}, employ self-supervised learning by using graph augmentation strategies to capture global and higher-order structural patterns. MUSE \cite{MUSE_Yuan_2023} builds semantic and contextual views for contrastive learning, integrating multi-view representations through a learned fusion controller.

While adjacency partitioning-based methods have shown strong theoretical and empirical performance on homophilous graphs, they do not extend naturally to heterophilous or directed settings. Conversely, self-supervised clustering approaches, though flexible, often lack explicit clustering objectives and/or still implicitly depend on homophily. In the following section, we introduce HeNCler, which bridges these gaps.

\section{Method}\label{sec:Method}
\myparagraph{Model motivation}
Our approach employs an RKM auto-encoder framework, which has been shown to be effective in unsupervised representation learning by jointly optimizing feature mappings and projection matrices within a kernel-based setting \cite{strkm}. We propose a weighted KSVD (wKSVD) loss that employs double feature mappings to learn an asymmetric similarity matrix. This \emph{learned} similarity alleviates the homophily assumption by capturing long-range dependencies and enables the modeling of \emph{asymmetric} relationships---making it particularly well-suited for directed graphs. Even in undirected graphs, this asymmetry can capture nuanced relational patterns. In fact, many graphs treated as undirected are inherently directed (e.g., citation networks). Modeling such asymmetries has been shown to be beneficial \cite{He2023 ,tao_nonlinear_2023}.
 Additionally, the wKSVD loss admits a spectral graph biclustering interpretation, offering further theoretical insight. 
 We first introduce the general wKSVD framework, followed by our HeNCler model, which operates in the primal setting and jointly learns the feature mappings in an end-to-end manner.

\subsection{Kernel spectral biclustering with asymmetric similarities}
Consider a dataset with two input sources $\{\bm{x}_i\}_{i=1}^n$ and $\{\bm{z}_j\}_{j=1}^m$, on which we want to define an unsupervised learning task.
To this end, we introduce a weighted kernel singular value decomposition model (wKSVD), starting from the following primal optimization problem, which is a weighted variant of the KSVD formulation:
\begin{small}
\begin{gather}
\underset{\bm{U},\bm{V},\bm{e},\bm{r}}{\min} J \triangleq \text{Tr}(\bm{U}^\top \bm{V}) - \frac{1}{2}\sum_{i=1}^{n} w_{1,i} \bm{e}_i^\top \boldsymbol{\Sigma}^{-1} \bm{e}_i - \frac{1}{2}\sum_{j=1}^{m} w_{2,j} \bm{r}_j^\top \boldsymbol{\Sigma}^{-1} \bm{r}_j \nonumber
    \\
    \text{s.t. }  \{\bm{e}_{i}=\bm{U}^\top\phi(\bm{x}_i), \ \forall i= 1,\dots, n; \quad \bm{r}_{j}=\bm{V}^\top\psi(\bm{z}_j), \forall j= 1,\dots, m\},
    \label{eq:primal_optimization}
\end{gather}
\end{small}
with projection matrices $\bm{U}, \bm{V} \in \mathbb{R}^{d_f\times s}$; strictly positive weighting scalars $w_{1,i}, w_{2,j}$; latent variables $\bm{e}_i, \bm{r}_j \in \mathbb{R}^s$; diagonal and positive definite hyperparameter matrix $\boldsymbol{\Sigma}\in\mathbb{R}^{s \times s}$; and centered feature maps $\phi(\cdot):\mathbb{R}^{d_x}\mapsto\mathbb{R}^{d_f}$ and $\psi(\cdot):\mathbb{R}^{d_z}\mapsto\mathbb{R}^{d_f}$.
Details on centering of the feature maps are provided in the supplementary materials.
The following derivation shows the equivalence with the spectral biclustering problem.
\begin{proposition}\label{prop: wKSVD duality}
The solution to the primal problem (\ref{eq:primal_optimization}) can be obtained by solving the singular value decomposition of
\begin{small}
\begin{equation}
    \bm{W}_1^{1/2} \bm{S} \bm{W}_2^{1/2} = \bm{H}_e \bm{\Sigma} \bm{H}_r^\top, \label{eq: SVD}
\end{equation}
\end{small}
 where $\bm{W}_1$ and $\bm{W}_2$ are diagonal matrices such that $W_{1,ii} = w_{1,i}$ and $W_{2,jj} = w_{2,j}$, $\bm{S}=\boldsymbol{\Phi \Psi}^\top$ is an asymmetric similarity matrix where $S_{ij} = \phi({\bm{x}_i})^\top\psi({\bm{z}_j})$, $\boldsymbol{\Phi} = [\phi({\bm{x}_1}) \dots \phi({\bm{x}_n})]^\top$, $\boldsymbol{\Psi} = [\psi({\bm{z}_1}) \dots \psi({\bm{z}_m})]^\top$, and where $\bm{H}_e = [\bm{h}_{\bm{e}_1} ... \bm{h}_{\bm{e}_n}]^\top$, and $\bm{H}_r = [\bm{h}_{\bm{r}_1} ... \bm{h}_{\bm{r}_m}]^\top$ are the left and right singular vectors respectively; and by applying $\bm{r}_{j}=\boldsymbol{\Sigma}  \bm{h}_{\bm{r}_j} / \sqrt{w_{2,j}}$ and $\bm{e}_{i} = \boldsymbol{\Sigma}  \bm{h}_{\bm{e}_i} / \sqrt{w_{1,i}}$.
 
\end{proposition}
\begin{proof}
We now introduce dual variables $\bm{h}_{\bm{e}_i}$ and $\bm{h}_{\bm{r}_j}$ using a specific instance of the Fenchel–Young inequality \cite{conjugate_duality_book}:
\begin{small}
\begin{gather}
        \frac{1}{2}w_{1,i} \ \bm{e}_i^\top\boldsymbol{\Sigma}^{-1}\bm{e}_i+\frac{1}{2}\bm{h}_{\bm{e}_i}^\top\boldsymbol{\Sigma}\bm{h}_{\bm{e}_i}\geq \sqrt{w_{1,i}} \ \bm{e}_i^\top\bm{h}_{\bm{e}_i},\nonumber \\
    \frac{1}{2}w_{2,j} \ \bm{r}_j^\top\boldsymbol{\Sigma}^{-1}\bm{r}_j+\frac{1}{2}\bm{h}_{\bm{r}_j}^\top\boldsymbol{\Sigma}\bm{h}_{\bm{r}_j}\geq \sqrt{w_{2,j}} \ \bm{r}_j^\top\bm{h}_{\bm{r}_j},\label{eq:FY}
\end{gather}
\end{small}
$\forall \bm{e}_i, \bm{r}_j,\bm{h}_{\bm{e}_i}, \bm{h}_{\bm{r}_j}\in\mathbb{R}^s; \ \forall w_{1,i}, w_{2,j} \in \mathbb{R}_{>0};$ and $\forall \boldsymbol{\Sigma} \in \mathbb{R}^{s\times s}_{\succ 0}$.\footnote{Refer to the supplementary materials for verification of these inequalities.}

By substituting the constraints of (\ref{eq:primal_optimization}) and inequalities (\ref{eq:FY}) into the objective function of (\ref{eq:primal_optimization}), we obtain an objective in primal and dual variables as an upper bound on the primal objective $\bar{J}\geq J$:
\begin{small}
\begin{multline}
\min_{\bm{U},\bm{V},\bm{h}_e,\bm{h}_r}\bar{J} \triangleq \text{Tr}(\bm{U}^\top \bm{V})
    -\sum_{i=1}^{n} \sqrt{w_{1,i}}\ \phi(\bm{x}_i)^\top  \bm{U} \bm{h}_{\bm{e}_i} +\frac{1}{2}\sum_{i=1}^{n}\bm{h}_{\bm{e}_i}^\top \boldsymbol{\Sigma}\bm{h}_{\bm{e}_i}\\
   -\sum_{j=1}^{m} \sqrt{w_{2,j}}\ \psi(\bm{z}_j)^\top  \bm{V} \bm{h}_{\bm{r}_i} +\frac{1}{2}\sum_{j=1}^{m}\bm{h}_{\bm{r}_j}^\top \boldsymbol{\Sigma}\bm{h}_{\bm{r}_j}.\label{eq:AKSC_energy}
\end{multline}\end{small}
Next, we formulate the stationarity conditions of problem (\ref{eq:AKSC_energy}):
\begin{small}\begin{equation}
    \begin{array}{ccc}
        \dfrac{\partial \bar{J}}{\partial \bm{V}}=0 \Rightarrow \ \bm{U}=\sum_{j=1}^m \sqrt{w_{2,j}}\ \psi(\bm{z}_{j})\bm{h}_{\bm{r}_j}^\top,
        & 
        \quad
        &
        \dfrac{\partial \bar{J}}{\partial \bm{h}_{\bm{e}_i}}=0 \Rightarrow \ 
        \boldsymbol{\Sigma}  \bm{h}_{\bm{e}_i} = \sqrt{w_{1,i}}\ \bm{U}^\top\phi(\bm{x}_{i}),\\
        &&\\
        \dfrac{\partial \bar{J}}{\partial \bm{U}}=0 \Rightarrow \ \bm{V}=\sum_{i=1}^n \sqrt{w_{1,i}}\ \phi(\bm{x}_{i})\bm{h}_{\bm{e}_i}^\top,
        &
        \quad
        & 
        \dfrac{\partial \bar{J}}{\partial \bm{h}_{\bm{r}_j}}=0 \Rightarrow \ \boldsymbol{\Sigma}  \bm{h}_{\bm{r}_j} = \sqrt{w_{2,j}}\ \bm{V}^\top\psi(\bm{z}_{j}),
    \end{array}\label{eq:SC}
\end{equation}\end{small}
from which we then eliminate the primal variables $\bm{U}$ and $\bm{V}$. This yields:

\begin{small}\begin{equation}
    \left[\begin{array}{cc}
        \bm{0} & \bm{W}_1^{1/2} \bm{S} \bm{W}_2^{1/2} \\
         \bm{W}_2^{1/2} \bm{S}^\top \bm{W}_1^{1/2} & \bm{0}
    \end{array}\right]
    \left[\begin{array}{c}
        \bm{\bm{H}_e} \\
         \bm{\bm{H}_r}
    \end{array}\right]
    =
        \left[\begin{array}{c}
        \bm{\bm{H}_e} \\
         \bm{\bm{H}_r}
    \end{array}\right]
    \boldsymbol{\Sigma},\label{eq:dual wKSVD}
\end{equation}\end{small}

where $\bm{0}$ is an all-zeros matrix. The above eigenvalue problem is equivalent with (\ref{eq: SVD}), and the stationarity conditions (\ref{eq:SC}) provide the relationships between primal and dual variables, which concludes the proof.
\qed
\end{proof}
\begin{figure*}[t]
  \centering 
  \resizebox{1\textwidth}{!}{
 \begin{tikzpicture}[
PR/.style={rectangle, draw=red!60, fill=red!5, very thick, minimum size=15mm},
DU/.style={rectangle, draw=blue!60, fill=blue!5, very thick, minimum size=15mm},]
\node[PR]    (Input)                              {$\mathcal{G}(\bm{X}, \bm{A})$};
\node[PR]    (Feature)      [right= 4 of Input] {$\begin{array}{c}
     \bm{\phi}_v  \\
    \bm{\psi}_v 
\end{array}$};

\node (mid) [right=2.0 of Feature] {};
\node [above=1.6 of mid] {$\text{MLP}_{rec}(\bm{Ue}_v || \bm{Ve}_v;\bm{\theta}_{\rm rec})$};
\node [below=1.65 of mid] {$\sigma(\bm{e}_u^\top \bm{U}^\top\bm{V} \bm{r}_v)$};
\node[PR]    (Latent)       [right= 4 of Feature] {$\begin{array}{c}
     \bm{e}_v  \\
    \bm{r}_v 
\end{array}$};
\node[DU] (S) [below=2.3 of Feature] {$\bm{S}=\bm{\Phi\Psi}^\top$};
\node[DU] (h) [below=2.3 of Latent] {$\begin{array}{c}
     \bm{h}_{\bm{e}_v}  \\
\bm{h}_{\bm{r}_v} 
\end{array}$};
\draw[->, very thick] (Input.east)  to node[below] {$\text{MLP}_{\psi}(\bm{x}_v || {\rm PE}_v;\bm{\theta}_\psi)$} (Feature.west);
\draw[->, very thick] (Feature.east)  to node[below] {$\bm{V}$} (Latent.west);
\draw[->, very thick] (Input.east)  to node[above] {$\text{MLP}_{\phi}(\bm{x}_v || {\rm PE}_v;\bm{\theta}_\phi)$} (Feature.west);
\draw[->, very thick] (Feature.east)  to node[above] {$\bm{U}$} (Latent.west);
\draw[->, thick, dashed] (Latent.north) .. controls  +(up:15mm) and +(up:15mm)   .. (Input.north);
\draw[->, thick, dashed] (Latent.south) .. controls  +(down:15mm) and +(down:15mm)   .. (Input.south);
\draw[->, very thick] (S.east)  to node[above] {$\text{SVD}(\bm{D}_1^{-1/2}\bm{S}\bm{D}_2^{-1/2})$} (h.west);
\draw[<->, thick, purple] (Feature.south) to node[left] {$\begin{array}{c}
     \text{Primal}  \\
    \text{Dual} 
\end{array}$} (S.north);
\draw[<->, thick, purple] (Latent.south) to node[left] {$\begin{array}{c}
     \text{Primal}  \\
    \text{Dual} 
\end{array}$} (h.north);
\draw[<->, thick, purple] (Feature.south) to (S.north);
\end{tikzpicture}
}
 \caption{\textbf{The HeNCler model}. HeNCler operates in the primal setting (top of the figure in red) and uses a double multilayer perceptron (MLP) to map node representations into a feature space. The obtained representations $\bm{\phi}_v$ and $\bm{\psi}_v$ are then projected into a latent space, yielding $\bm{e}_v$ and $\bm{r}_v$ respectively. The wKSVD loss ensures that these latent representations correspond to the dual equivalent (bottom of the figure in blue) i.e., a biclustering of the asymmetric similarity graph defined by $\bm{S}$.
  The node and edge reconstructions (dashed arrows) aid in the feature map learning.} \label{fig:model}
\end{figure*}

We have thus shown the connection between the primal (\ref{eq:primal_optimization}) and dual formulation (\ref{eq: SVD}). Similarly to the KSVD framework, the wKSVD framework can be used for learning with asymmetric kernel functions and/or rectangular data sources. The spectral biclustering problem can now easily be obtained by choosing the weights $w_{1,i}$ and $w_{2,j}$ appropriately.
\begin{corollary}\label{cor: KSBC}
Given Proposition \ref{prop: wKSVD duality}, and by choosing $\bm{W}_1$ and  $\bm{W}_2$ to equal $\bm{D}^{-1/2}_{1}$ and  $\bm{D}^{-1/2}_{2}$, where $D_{1,ii}=\sum_j S_{ij}$ and  $D_{2,jj}=\sum_i S_{ij}$, we obtain the random walk interpretation $\bm{D}_1^{-1/2} \bm{S} \bm{D}_2^{-1/2} = \bm{H}_e \bm{\Sigma} \bm{H}_r^\top$ of the spectral graph bipartitioning problem for the bipartite graph $\mathcal{S}=(\bm{\Phi},\bm{\Psi},\bm{S})$.
\end{corollary}
Moreover, the wKSVD framework is more general as, on the one hand, one can use a given similarity matrix (e.g. adjacency matrix of a graph) or (asymmetric) kernel function in the dual, or, on the other hand, one can  choose to use explicitly defined (deep) feature maps in both primal or dual. 
\subsection{The HeNCler model}
HeNCler employs the wKSVD framework in a graph setting, where the dataset is a node set $\mathcal{V}$ and where the asymmetry arises from employing two different mappings that operate on the nodes given the entire graph $\mathcal{G=(\bm{X},\bm{A}}$). Our approach is visualized in Figure \ref{fig:model}.

In the preceding subsection, we showed that the primal wKSVD formulation \eqref{eq:primal_optimization} has an equivalent dual problem corresponding to the graph bipartitioning problem. This equivalence holds when $w_{1,i}$ and $w_{2,j}$ are chosen to equal the square root of the inverse of the out-degree and in-degree of a similarity graph $\mathcal{S}$ respectively. The similarity graph $\mathcal{S}$ depends on the feature mappings $\phi(\cdot)$ and $\psi(\cdot)$, which for our method not only depend on the node of interest, but also on the rest of the input graph and the learnable parameters. The mappings for a node $v$ thus become $\phi(\bm{x}_v, \mathcal{G}; \bm{\theta_\phi})$ and $\psi(\bm{x}_v, \mathcal{G}; \bm{\theta_\psi})$ and we will ease these notations to $\phi(\bm{x}_v)$ and $\psi(\bm{x}_v)$. The ability of our method to learn these feature mappings is an important aspect of our contribution, as a key motivation is that we need to learn new similarities for clustering heterophilous graphs.

HeNCler's loss function is comprised of three terms: the wKSVD-loss, a node-reconstruction loss, and an edge-reconstruction loss: 
\begin{small}\begin{equation*}
        \mathcal{L}_{\rm wKSVD}(\bm{U}, \bm{V}, \bm{\Sigma}, \bm{\theta}_\phi, \bm{\theta}_\psi) + \mathcal{L}_{\rm NodeRec}(\bm{U}, \bm{V}, \bm{\theta}_\phi, \bm{\theta}_\psi, \bm{\theta}_{\rm rec}) + \mathcal{L}_{\rm EdgeRec}(\bm{U}, \bm{V}, \bm{\theta}_\phi, \bm{\theta}_\psi),
\end{equation*}\end{small}
where the trainable parameters of the model are in the the multilayer perceptron (MLP) feature maps ($\bm{\theta}_\phi$ and $\bm{\theta}_\psi$), the MLP node decoder ($\bm{\theta}_{\rm rec}$), the $\bm{U}$ and $\bm{V}$ projection matrices, and the singular values $\bm{\Sigma}$. All these parameters are trained end-to-end and we next explain the losses in more detail.

\myparagraph{wKSVD-loss}
Instead of solving the SVD in the dual formulation, HeNCler leverages the primal formulation (\ref{eq:primal_optimization}) of the wKSVD framework for greater computational efficiency. While equation (\ref{eq:primal_optimization}) assumes that the feature maps $\phi(\cdot)$ and $\psi(\cdot)$ are fixed, HeNCler utilizes parametric functions $\phi(\cdot; \bm{\theta}_\phi)$ and $\psi(\cdot; \bm{\theta}_\psi)$, enabling it to learn new similarities between nodes. By incorporating regularization terms for these functions and defining the weighting scalars as $w_{1,v} = D_{1,vv}^{-1} = 1/\sum_u \phi(\bm{x}_v)^\top \psi(\bm{x}_u)$ and $w_{2,v} = D_{2,vv}^{-1} = 1/\sum_u \phi(\bm{x}_u)^\top \psi(\bm{x}_v)$, we derive the wKSVD-loss:
\begin{small}\begin{multline}
    \mathcal{L}_{\rm wKSVD} \triangleq
    -\sum_{v=1}^{|\mathcal{V}|} D_{1,vv}^{-1} \ \phi(\bm{x}_v)^\top{\bm{U}}\bm{\Sigma}^{-1}{\bm{U}}^\top\phi(\bm{x}_v) 
    - \sum_{v=1}^{|\mathcal{V}|} D_{2,vv}^{-1} \ \psi(\bm{x}_v)^\top{\bm{V}}\bm{\Sigma}^{-1}{\bm{V}}^\top\psi(\bm{x}_v)\\
    + \text{Tr}({\bm{U}}^\top {\bm{V}}) + \sum_{v=1}^{|\mathcal{V}|}\sqrt{D_{1,vv}^{-1} \ D_{2,vv}^{-1}} \ \phi(\bm{x}_v)^\top\psi(\bm{x}_v).
    \label{eq:J_hencler}
\end{multline}\end{small}

The primal formulation of HeNCler (\ref{eq:J_hencler}) consists of four terms with two distinct objectives. The first two terms promote the weighted variance of the learned node representations $\bm{e}$ and $\bm{r}$, encouraging informative embeddings. The third and fourth terms serve as regularizers that enforce asymmetry by penalizing the similarity between $\bm{U}$ and $\bm{V}$, and between $\phi(\bm{x}_v)$ and $\psi(\bm{x}_v)$, respectively.

For the two feature maps $\phi(\cdot)$ and $\psi(\cdot)$, we employ two MLPs: $\phi(\bm{x}_v, \mathcal{G}; \bm{\theta}_\phi) \equiv \text{MLP}_{\phi}(\bm{x}_v || {\rm PE}_v; \bm{\theta}_\phi)$ and $\psi(\bm{x}_v, \mathcal{G}; \bm{\theta}_\psi) \equiv {\rm MLP}_{\psi}(\bm{x}_v || {\rm PE}_v; \bm{\theta}_\psi)$. We construct a random walks positional encoding (PE) \cite{dwivedi2022graph} to embed the network's structure and concatenate this encoding with the node attributes. The MLPs have two linear layers with a LeakyReLU activation function in between, followed by a batch normalization layer. The singular values in $\bm{\Sigma}$ are jointly learned, constrained to lie between $0$ and $1$, with the additional condition that $\text{Tr}(\bm{\Sigma}^{-\frac{1}{2}}) = 1$.

\myparagraph{Reconstruction losses}
Since the feature maps $\phi(\cdot)$ and $\psi(\cdot)$ need to be learned, an additional loss function beyond the above regularization term is required to effectively optimize the parameters of the MLPs.
As the node clustering setting is completely unsupervised, we add a decoder network and a reconstruction loss. This technique has been proven to be effective for unsupervised learning in the RKM-framework \cite{strkm}, as well as for unsupervised node representation learning \cite{sun_dual-decoder_2021}. For heterophilous graphs, we argue that it is particularly important to also reconstruct node features and not only the graph structure.

For the node reconstruction, we first project the $\bm{e}$ and $\bm{r}$ variables back to feature space, concatenate these and then map to input space with another MLP. This MLP has also two layers and a leaky ReLU activation function. The hidden layer size is set to the average of the latent dimension and input dimension.
With the mean-squared-error as the associated loss, this gives:
\begin{small}
\begin{equation*}
       \mathcal{L}_{\rm NodeRec} = \frac{1}{|\mathcal{V}|}\sum_{v \in \mathcal{V}}  ||{\rm MLP_{rec}}(\bm{U} \bm{e}_v || \bm{V} \bm{r}_v; \bm{\theta}_{\rm rec}) - \bm{x}_v||^2. 
\end{equation*}
\end{small}
To reconstruct edges, we use a simple dot-product decoder $\sigma(\bm{e}_u^\top \bm{U}^\top\bm{V} \bm{r}_v)$ where $\sigma$ is the sigmoid function. By using the $\bm{e}$ representation for source nodes and $\bm{r}$ for target nodes, this reconstruction is asymmetric and can reconstruct directed graphs. We use a binary cross-entropy loss:  
\begin{small}
$$
        \mathcal{L}_{\rm EdgeRec} =  \frac{1}{|\mathcal{U}|}\sum_{(u,v) \in \mathcal{U}} \text{BCE}(\sigma(\bm{e}_u^\top \bm{U}^\top\bm{V} \bm{r}_v), \mathcal{E}_{uv}),
$$
\end{small}
where $\mathcal{U}$ is a node-tuple set, resampled every epoch, containing $2|\mathcal{V}|$ positive edges from $\mathcal{E}$ and $2|\mathcal{V}|$ negative edges from $\mathcal{E}^C$, and $\mathcal{E}_{uv} \in \{0,1\}$ indicates whether an edge $(u,v)$ exist: $(u,v)\in\mathcal{E}$.

\myparagraph{Constraints and cluster assignments} The batch normalization in the MLPs keeps the wKSVD-loss bounded and the constraints on the singular values is enforced by a softmax function. Cluster assignments are obtained by KMeans clustering on the concatenation of learned $\bm{e}$ and $\bm{r}$ node representations.

{HeNCler} jointly learns the wKSVD projection matrices, ${\bm{U}}$ and ${\bm{V}}$, along with the feature map parameters, $\bm{\theta_\phi}$ and $\bm{\theta_\psi}$. The wKSVD loss improves the cluster-ability of the learned similarity graph, ensuring that $\bm{e}$ and $\bm{r}$ function as spectral biclustering embeddings. The two distinct feature maps enable asymmetric learning, effectively capturing potential asymmetric relationships in the data, while the reconstruction losses ensure robust and meaningful representation learning.
\section{Experiments}\label{sec:exp}
\myparagraph{Datasets} We assess the performance of HeNCler on heterophilous attributed graphs that are available in literature. 
We use Texas, Cornell, and Wisconsin \cite{geom_gcn_2020}, which are directed webpage networks where edges encode hyperlinks between pages.
Next, we use Chameleon and Squirrel \cite{data_wiki}, which are undirected Wikipedia webpage networks where edges encode mutual links.  
We further evaluate our model on three undirected graphs: Roman-Empire, Minesweeper, and Tolokers \cite{platonov2023a}. These datasets represent, respectively, a graph-structured Wikipedia article, a grid graph inspired by the Minesweeper game, and a crowdsourcing interaction network. More details are provided in the supplementary materials.

\myparagraph{Model selection and metrics} Model selection in this unsupervised setting is inherently challenging, as the most appropriate evaluation metric depends on the downstream task. Consequently, model selection is beyond the scope of this paper. Instead, we evaluate our model in a task-agnostic manner and ensure a fair comparison with the baselines. We fix the hyperparameter configuration of the models across all datasets. We train for a fixed number of epochs and keep track of the evaluation metrics to report the best observed result. We repeat the training process 10 times and report average best results with standard deviations. 
We report the normalized mutual information (NMI) and pairwise F1-scores, based on the class labels. 

\myparagraph{Baselines and hyperparameters} We compare our model against several methods, including a simple KMeans based on node attributes, adjacency partitioning-based approaches such as MinCutPool \cite{bianchi_spectral_2020} and DMoN \cite{tsitsulin_graph_2023}, as well as S$^3$GC \cite{devvrit_s3gc_2022}, MUSE \cite{MUSE_Yuan_2023}, and HoLe \cite{gu2023homophily}, which represent the current state-of-the-art in homophilous and heterophilous node clustering.
For HeNCler, we fix the hyperparameters as follows: MLP hidden dimensions $256$, output dimensions $128$, latent dimension $s = 2 \times \#\text{classes}$, learning rate $0.01$, and epochs $300$.
For the baselines, we used their code implementations and the default hyperparameter settings as proposed by the authors. The number of clusters to infer is set to the number of classes for all methods.

\myparagraph{Experimental results} Table \ref{tab:hetero_exp} presents the experimental results for heterophilous graphs. HeNCler consistently demonstrates superior performance, significantly outperforming the baselines, especially on directed graphs. For undirected graphs, HeNCler also shows strong results, achieving the best performance in 5 out of 10 cases. These results highlight HeNCler's versatility and effectiveness in handling heterophilous graph structures.
\begin{table}[]
\caption{Experimental results on heterophilous graphs. We report NMI and F1 scores for 10 runs (mean $\pm$ standard deviation), where higher values indicate better performance. The best results for each metric are highlighted in bold.  OOM indicates an out-of-memory error on a 16GB GPU.}
    \centering
    \begin{sc}
    \begin{scriptsize}
    \begin{tabular}{lcccccccc}
    \toprule
    & & \multicolumn{6}{c}{Baselines} & Ours \\ 
    \cmidrule(l){3-8} \cmidrule(l){9-9}
    \multicolumn{2}{l}{Dataset} & KMeans & MinCutPool & DMoN & S$^3$GC & MUSE & HoLe & HeNCler \\
    \midrule
    \midrule
    tex & NMI & 4.97\tiny{$\pm$}1.00 & 11.60\tiny{$\pm$2.19} & 9.06\tiny{$\pm$2.11} & 11.56\tiny{$\pm$1.46} & 39.23\tiny{$\pm$4.91}& 7.51\tiny{$\pm$0.37} & \textbf{43.65}\tiny{$\pm$2.52}\\
     & F1 & 59.27\tiny{$\pm$}0.83 & 55.26\tiny{$\pm$0.56} & 47.76\tiny{$\pm$4.79} & 43.69\tiny{$\pm$2.74} & 65.96\tiny{$\pm$3.52}& 35.30\tiny{$\pm$0.44} & \textbf{71.39}\tiny{$\pm$2.16}\\
    \midrule
     corn & NMI & 5.42\tiny{$\pm$}2.04 & 17.04\tiny{$\pm$1.61} & 12.49\tiny{$\pm$2.51} & 14.48\tiny{$\pm$1.79} & 38.99\tiny{$\pm$2.73}& 14.23\tiny{$\pm$0.17} & \textbf{41.52}\tiny{$\pm$4.35}\\
     & F1 & 52.97\tiny{$\pm$}0.24 & 51.21\tiny{$\pm$5.06}  & 43.83\tiny{$\pm$6.23} & 33.13\tiny{$\pm$0.83} & 60.58\tiny{$\pm$3.61}& 34.64\tiny{$\pm$0.56} & \textbf{63.40}\tiny{$\pm$3.67}\\
    \midrule
     wis & NMI & 6.84\tiny{$\pm$}4.39 & 13.38\tiny{$\pm$2.36} & 12.56\tiny{$\pm$1.23} & 13.07\tiny{$\pm$0.61} & 39.71\tiny{$\pm$2.22}& 11.89\tiny{$\pm$0.42} & \textbf{47.13}\tiny{$\pm$1.76}\\
     & F1 & 56.16\tiny{$\pm$}0.58 & 55.63\tiny{$\pm$2.96} & 45.72\tiny{$\pm$7.85} & 31.71\tiny{$\pm$2.25} & 58.94\tiny{$\pm$3.09}& 37.05\tiny{$\pm$0.02} & \textbf{68.30}\tiny{$\pm$2.17}\\
    \midrule
     cha & NMI & 0.44\tiny{$\pm$}0.11 & 11.88\tiny{$\pm$1.99} & 12.87\tiny{$\pm$1.86} & 15.83\tiny{$\pm$0.26} & 23.06\tiny{$\pm$0.28}& 8.76\tiny{$\pm$0.15}& \textbf{23.89}\tiny{$\pm$0.84} \\
     & F1 & \textbf{53.23}\tiny{$\pm$}0.07 & {50.40}\tiny{$\pm$5.65} & {45.05}\tiny{$\pm$4.30} & 36.51\tiny{$\pm$0.24} & 52.10\tiny{$\pm$0.48}& 30.61\tiny{$\pm$0.09}& 44.14\tiny{$\pm$1.83}\\
    \midrule
     squi & NMI & 1.40\tiny{$\pm$}2.12 & 6.35\tiny{$\pm$0.32} & 3.08\tiny{$\pm$0.38} & 3.83\tiny{$\pm$0.11} & 8.30\tiny{$\pm$0.23}& 4.99\tiny{$\pm$0.09}& \textbf{9.67}\tiny{$\pm$0.13}\\
    & F1 & 54.05\tiny{$\pm$}2.72 & \textbf{55.26}\tiny{$\pm$0.57} & 49.21\tiny{$\pm$2.74} & 35.08\tiny{$\pm$0.18} & 50.07\tiny{$\pm$5.99}& 28.71\tiny{$\pm$0.36}& 36.51\tiny{$\pm$2.39} \\
    \midrule
     rom & NMI & 35.20\tiny{$\pm$}1.79 & 9.97\tiny{$\pm$2.02} & 13.14\tiny{$\pm$0.53} & 14.48\tiny{$\pm$0.21} & \textbf{40.50}\tiny{$\pm$0.73}& \multirow{2}{*}{OOM} & 36.99\tiny{$\pm$0.61}\\
     & F1 & 37.17\tiny{$\pm$}2.12 & \textbf{42.19}\tiny{$\pm$0.26} & 22.69\tiny{$\pm$3.91} & 17.76\tiny{$\pm$0.53} & 38.34\tiny{$\pm$0.35}& & 35.43\tiny{$\pm$1.07}\\
    \midrule
     mine & NMI & 0.02\tiny{$\pm$}0.02  & {6.16}\tiny{$\pm$2.17} & \textbf{6.87}\tiny{$\pm$2.91} & {6.53}\tiny{$\pm$0.17} & 0.06\tiny{$\pm$0.01}& 6.46\tiny{$\pm$0.07}& 0.06\tiny{$\pm$0.00}\\
     & F1 & 73.63\tiny{$\pm$}3.58 & 71.76\tiny{$\pm$8.86} & 70.42\tiny{$\pm$9.47} & 48.78\tiny{$\pm$0.63} & 75.77\tiny{$\pm$2.24}& 63.01\tiny{$\pm$0.13}& \textbf{76.48}\tiny{$\pm$1.56}\\
    \midrule
     tol & NMI & 3.04\tiny{$\pm$}2.83 & 6.68\tiny{$\pm$0.98} & 6.69\tiny{$\pm$0.20} & 5.99\tiny{$\pm$0.05} & 6.67\tiny{$\pm$0.55}& 5.14\tiny{$\pm$0.06} & \textbf{6.73}\tiny{$\pm$0.59}\\
     & F1 & 65.56\tiny{$\pm$}10.49 & {72.10}\tiny{$\pm$10.38} & 67.87\tiny{$\pm$4.74} & 59.17\tiny{$\pm$0.27} & 73.56\tiny{$\pm$1.94}& 66.35\tiny{$\pm$1.12} & \textbf{73.66}\tiny{$\pm$2.10} \\
    \bottomrule
    \end{tabular}
    \end{scriptsize}
    \end{sc}
    \label{tab:hetero_exp}
\end{table}

\myparagraph{Additional experiments} are presented in the supplementary materials. We include an ablation study, some experiments on homophilous graphs, a visualization of the learned asymmetries, and a computational complexity analysis.

\myparagraph{Discussion} A key motivation behind HeNCler is to learn a new graph representation where nodes belonging to the same cluster are positioned closer together, driven by the clustering objective. This results in spectral biclustering embeddings that exhibit improved cluster-ability. As HeNCler uses KMeans to obtain cluster assignments, the comparisons between HeNCler and KMeans demonstrate that our model enhances the cluster-ability of the node representations relative to the original input features.

The asymmetry in HeNCler eliminates the undirected constraints of traditional adjacency partitioning-based models, enabling superior performance on directed graphs. Furthermore, our ablation study in the supplementary materials shows that, while most of the performance on undirected graphs stems from the graph learning component, HeNCler is able to capture and learn additional meaningful asymmetric information. This capacity to extract valuable asymmetric insights from symmetric data is a common occurrence in KSVD frameworks \cite{He2023, tao_nonlinear_2023}. Importantly, thanks to the added performance boost from asymmetry, on top of the benefits from similarity learning, HeNCler outperforms state-of-the-art models, even when applied to undirected graphs.

While our experiments focus on benchmark heterophilous graphs, the design of HeNCler makes it a promising candidate for domains characterized by relational asymmetry and weak homophily. For example, in community detection within directed social or communication networks, HeNCler’s ability to model asymmetric relationships could help capture influence flows that do not align with node similarity. In recommender systems, HeNCler could be applied to bipartite user-item graphs by using separate feature maps for each node type. This would enable the construction of a learned bipartite similarity graph that aligns user preferences with item characteristics in an unsupervised manner. Similarly, in biological networks such as gene regulatory systems, the model’s capacity to handle directed, heterophilous interactions could prove valuable.

\section{Conclusion and future work}\label{sec:concl}
We tackle three limitations of current node clustering algorithms that prevent these methods from effectively clustering nodes in heterophilous graphs: they assume homophily in their loss, they are only defined for undirected graphs and/or they lack a specific focus on clustering.

To this end, we introduce a weighted kernel SVD framework and harness its primal-dual equivalences. HeNCler relies on the dual interpretation for its theoretical motivation, while it benefits from the computational advantages of its implementation in the primal. It learns new similarities, which are asymmetric where necessary, and node embeddings resulting from the spectral biclustering interpretation of these learned similarities. As empirical evidence shows, our approach effectively eliminates the aforementioned limitations, significantly outperforming current state-of-the-art alternatives.

Future work could explore integrating contrastive learning into HeNCler, potentially combining the strengths of both paradigms. Another direction is to investigate cluster assignments in a graph pooling setting (i.e., differentiable graph coarsening), enabling end-to-end training for downstream graph-level tasks. Additionally, applying HeNCler to real-world domains such as recommender systems, biological networks, or directed social graphs—where asymmetric and heterophilous structures naturally occur—would be an important step toward validating its broader applicability.
\begin{credits}
\subsubsection{\ackname} ESAT-STADIUS has received funding from the European Research Council under the European Union's Horizon 2020 research and innovation program / ERC Advanced Grant E-DUALITY (787960). This paper reflects only the authors' views and the Union is not liable for any use that may be made of the contained information. ESAT-STADIUS received funding from the Flemish Government (AI Research Program); iBOF/23/064; KU Leuven C1 project C14/24/103. Johan Suykens and Sonny Achten are also affiliated to Leuven.AI - KU Leuven institute for AI, B-3000, Leuven, Belgium.

LIONS-EPFL was supported by Hasler Foundation Program: Hasler Responsible AI (project number 21043), by the ARO under Grant Number W911NF-24-1-0048, and by the Swiss National Science Foundation (SNSF) under grant number 200021\_205011.

\subsubsection{\discintname}
The authors have no competing interests to declare that are relevant to the content of this article.
\end{credits}
%
%
%
\bibliographystyle{splncs04}
\bibliography{references}   
\appendix
\section{Note on feature map centering}\label{app:centering}
In the wKSVD framework, we assume that the feature maps are centered. More precisely, given two arbitrary mappings $\phi(\cdot)$ and $\psi(\cdot)$, the centered mappings are obtained by subtracting the weighted mean:
\begin{gather*}
    \phi_c(\bm{x}_i) =  \phi(\bm{x}_i) - \frac{\sum_{k=1}^{n} w_{1,k} \ \phi(\bm{x}_k)}{\sum_{k=1}^{n} w_{1,k}}, \\
    \psi_c(\bm{z}_j) =  \psi(\bm{z}_j) - \frac{\sum_{l=1}^{m} w_{2,l} \ \psi(\bm{z}_l)}{\sum_{l=1}^{m} w_{2,l}}.
\end{gather*}
Although we use the primal formulation in this paper, we next show how to obtain this centering in the dual for the sake of completeness. When using a kernel function or a given similarity matrix, one has no access to the explicit mappings and has to do an equivalently centering in the dual using:
\begin{equation*}
    \bm{S}_c = \bm{M}_1 \bm{S} \bm{M}_2^\top,
\end{equation*}
where $\bm{M}_1$ and $\bm{M}_2$ are the centering matrices:
\begin{gather*}
\bm{M}_1 = \bm{I}_n - \frac{1}{\bm{1}_n^\top\bm{W}_1\bm{1}_n}\bm{1}_n\bm{1}_n^\top\bm{W}_1,\\
\bm{M}_2 = \bm{I}_m - \frac{1}{\bm{1}_m^\top\bm{W}_2\bm{1}_m}\bm{1}_m\bm{1}_m^\top\bm{W}_2,
\end{gather*}
with $\bm{I}_n$ and $\bm{1}_n$ a $n\times n$ identity matrix and a $n$-dimensional all-ones vector respectively. We omit the subscript $c$ in the paper and assume the feature maps are always centered. Note that this can easily be achieved in the implementations by using the above equations.

\section{Fenchel-Young Inequalities}

\begin{small}
\begin{gather}
        \frac{1}{2}w_{1,i} \ \bm{e}_i^\top\boldsymbol{\Sigma}^{-1}\bm{e}_i+\frac{1}{2}\bm{h}_{\bm{e}_i}^\top\boldsymbol{\Sigma}\bm{h}_{\bm{e}_i}\geq \sqrt{w_{1,i}} \ \bm{e}_i^\top\bm{h}_{\bm{e}_i},\nonumber \\
    \frac{1}{2}w_{2,j} \ \bm{r}_j^\top\boldsymbol{\Sigma}^{-1}\bm{r}_j+\frac{1}{2}\bm{h}_{\bm{r}_j}^\top\boldsymbol{\Sigma}\bm{h}_{\bm{r}_j}\geq \sqrt{w_{2,j}} \ \bm{r}_j^\top\bm{h}_{\bm{r}_j},\label{eq:FY}
\end{gather}
\end{small}
The above inequalities can be verified by writing it in quadratic form: 

\begin{small}$	\frac{1}{2}
	\begin{bmatrix}
		\bm{e}^{\top}_{i} & \bm{h}^{\top}_{e_i} \\
	\end{bmatrix}
	\begin{bmatrix}
		w_{1,i}\bm{\Sigma}^{-1} & -\sqrt{w_{1,i}} \ \bm{I}_s          \\
		-\sqrt{w_{1,i}} \ \bm{I}_s     & \bm{\Sigma}
	\end{bmatrix}
	\begin{bmatrix}
		\bm{e}_{i} \\
		\bm{h}_{e_i}
	\end{bmatrix}
	\geq
	0, \quad \forall i
$ \end{small}
with $\bm{I}_s$ the $s$-dimensional identity matrix, which follows immediately from the Schur complement form: for a matrix $\bm{Q} =$ \begin{small}$\left[\begin{smallmatrix}
			\bm{Q}_1 & \bm{Q}_2 \\ \bm{Q}_2^{\top} & \bm{Q}_3 \end{smallmatrix}\right],$  \end{small}
one has $\bm{Q} \succeq 0$ if and only if $\bm{Q}_1 \succ 0$ and the Schur complement $\bm{Q}_3 - \bm{Q}_2^{\top} \bm{Q}_1^{-1} \bm{Q}_2 \succeq 0$.

\section{Additional Experiments}\label{app:additional_results}

\myparagraph{Homophilous experiments} Although our work primarily focuses on heterophilous graphs, we further evaluate our model on homophilous citation networks Cora, Citeseer, and PubMed \cite{data_pubmed2, data_pubmed1}. Relevant dataset statistics can be consulted in Table \ref{tab:homo_datasets}. We employ the same experimental setup as for the heterophilous datasets and report the experimental results in Table \ref{tab:homo_exp}. 
Although HoLe achieves the best overall performance--albeit at a higher computational cost--due to its emphasis on homophily enhancement, HeNCler surpasses adjacency partitioning methods such as MinCutPool and DMoN. Furthermore, HeNCler demonstrates competitive performance with MUSE, the state-of-the-art in heterophilous node clustering, further validating its robustness across diverse graph structures.
\begin{table*}[ht]
\caption{Dataset statistics.}
\begin{center}
\begin{tabular}{lcccccc}
\toprule
Dataset & short & \# Nodes & \# Edges & \# Classes & Directed & $\mathcal{H}(\mathcal{G})$\\
\midrule
Cora & cora & 2,708 & 5,278 & 7 & \myxmark & 0.765 \\
CiteSeer & cite & 3,327 & 4,614 & 6 & \myxmark & 0.627 \\
Pubmed & pub & 19,717 & 44,325 & 3 &\myxmark & 0.664 \\
Texas & tex & 183 & 325 & 5 & \mycheckmark & 0.000 \\
Cornell & corn & 183 & 298 & 5 & \mycheckmark & 0.150 \\
Wisconsin & wis &  251 & 515 & 5 & \mycheckmark & 0.084 \\
Chameleon & cha & 2,277 & 31,371 & 5 &\myxmark & 0.042 \\
Squirrel & squi & 5,201 & 198,353 & 5 &\myxmark & 0.031 \\
Roman-empire & rom & 22,662 & 32,927 & 18 & \myxmark & 0.021 \\
Minesweeper & mine & 10,000 & 39,402 & 2 & \myxmark & 0.009 \\
Tolokers & tol & 11,758 & 519,000 & 2 & \myxmark & 0.180 \\
\bottomrule
\end{tabular}
\end{center}
\label{tab:homo_datasets}
\end{table*} 
\begin{table}[]
\caption{Experimental results on homophilous graphs. We report NMI and F1 scores for 10 runs (mean $\pm$ standard deviation), where higher values indicate better performance. The best results for each metric are highlighted in bold.}
    \centering
    \begin{scriptsize}
    \begin{tabular}{lcccccccc}
    \toprule
    & & \multicolumn{6}{c}{Baselines} & Ours \\ 
    \cmidrule(l){3-8} \cmidrule(l){9-9}
    \multicolumn{2}{l}{Dataset} & KMeans & MinCutPool & DMoN & S$^3$GC & MUSE & HoLe & HeNCler \\
    \midrule
    \midrule
    cora & NMI & 35.0\tiny{$\pm$3.21} & 49.0\tiny{$\pm$2.24} & 51.7\tiny{$\pm$1.63} & 53.62\tiny{$\pm$0.55} & 36.45\tiny{$\pm$2.71}& \textbf{57.74}\tiny{$\pm$0.83} & 38.81\tiny{$\pm$2.26}\\
     & F1 & 36.0\tiny{$\pm$2.12} & 47.1\tiny{$\pm$1.78} & 51.8\tiny{$\pm$2.02} & 60.12\tiny{$\pm$0.46} & 50.78\tiny{$\pm$2.79} & \textbf{73.74}\tiny{$\pm$0.44} & 47.93\tiny{$\pm$2.60}\\
    \midrule
     cite & NMI & 19.9\tiny{$\pm$2.90} & 29.5\tiny{$\pm$3.21} & 30.3\tiny{$\pm$1.09} & \textbf{43.56}\tiny{$\pm$0.65} & 39.03\tiny{$\pm$1.99}& 43.41\tiny{$\pm$0.11} & 34.83\tiny{$\pm$2.21}\\
     & F1 & 39.4\tiny{$\pm$3.07} & 47.1\tiny{$\pm$1.21} & 57.4\tiny{$\pm$3.42} & 64.12\tiny{$\pm$0.28} & 52.89\tiny{$\pm$1.68}& \textbf{68.54}\tiny{$\pm$0.08} & 48.70\tiny{$\pm$2.79}\\
    \midrule
     pub & NMI & 31.4\tiny{$\pm$2.18} & 21.4\tiny{$\pm$1.46} & 25.7\tiny{$\pm$2.46} & 31.01\tiny{$\pm$2.35} & \textbf{36.09}\tiny{$\pm$3.26}& 31.29\tiny{$\pm$0.34} & 27.26\tiny{$\pm$1.72}\\
     & F1 & 59.2\tiny{$\pm$2.32} & 44.5\tiny{$\pm$2.47} & 34.3\tiny{$\pm$2.05} & 69.12\tiny{$\pm$1.39} & 61.26\tiny{$\pm$1.50}& \textbf{71.12}\tiny{$\pm$0.13} & 51.17\tiny{$\pm$1.75}\\
    \bottomrule
    \end{tabular}
    \end{scriptsize}
    \label{tab:homo_exp}
\end{table}

\myparagraph{Ablation studies} We conduct several ablation studies, presented in Table \ref{tab:full_ablation}. The `Undirected' variant refers to a simplified, symmetric version of the model that uses a single MLP for both the $\phi(\cdot)$ and $\psi(\cdot)$ mappings, i.e., $\phi(\cdot) \equiv \psi(\cdot)$. In this version, the model loses its asymmetry. The `wKSVD only' and `Reconstr only' variations reflect models that incorporate only the wKSVD loss ($\mathcal{L}_{\rm wKSVD}$) and the reconstruction losses ($\mathcal{L}_{\rm NodeRec} + \mathcal{L}_{\rm EdgeRec}$), respectively. Interestingly, as shown in Table \ref{tab:full_ablation}, even for undirected graphs, introducing asymmetry in HeNCler enhances clustering performance. Furthermore, all loss components are shown to contribute positively to HeNCler's overall performance.

\begin{table}[]
\caption{Ablation study results. We report NMI and F1 scores for 10 runs  (mean $\pm$ standard deviation in \%) where higher is better. Best results are highlighted in bold.}
    \centering
    \begin{scriptsize}
    \begin{tabular}{lccccc}
    \toprule
    & & \multicolumn{3}{c}{ablations} & full model \\ 
    \cmidrule(l){3-5} \cmidrule(l){6-6}
    \multicolumn{2}{l}{dataset} & Undirected & Reconstr only & wKSVD only & HeNCler \\
    \midrule
    \midrule
    tex & NMI & 27.58\tiny{$\pm$4.75} & 29.54\tiny{$\pm$2.27} & 31.64\tiny{$\pm$2.14} & \textbf{43.65}\tiny{$\pm$2.52}\\
     & F1 & 65.20\tiny{$\pm$2.06} & 66.64\tiny{$\pm$1.83} & 62.83\tiny{$\pm$3.91} & \textbf{71.39}\tiny{$\pm$2.16}\\
    \midrule
     corn & NMI & 18.12\tiny{$\pm$2.57} & 27.76\tiny{$\pm$3.29} & 20.63\tiny{$\pm$5.92} & \textbf{41.52}\tiny{$\pm$4.35}\\
     & F1  & 53.69\tiny{$\pm$0.98} & 54.70\tiny{$\pm$1.88} & 47.12\tiny{$\pm$2.61} & \textbf{63.40}\tiny{$\pm$3.67}\\
    \midrule
     wis & NMI & 25.08\tiny{$\pm$3.54} & 34.65\tiny{$\pm$1.86} & 39.86\tiny{$\pm$4.63} & \textbf{47.13}\tiny{$\pm$1.76}\\
     & F1 &  57.13\tiny{$\pm$1.34} & 62.28\tiny{$\pm$1.58} & 63.60\tiny{$\pm$2.46}  & \textbf{68.30}\tiny{$\pm$2.17}\\
    \midrule
     cha & NMI & 19.91\tiny{$\pm$0.48} & 22.02\tiny{$\pm$0.25} & 22.60\tiny{$\pm$0.57} & \textbf{23.89}\tiny{$\pm$0.84} \\
     & F1 &  44.08\tiny{$\pm$1.79} & 43.42\tiny{$\pm$1.63} & 42.98\tiny{$\pm$0.37} & \textbf{44.14}\tiny{$\pm$1.83}\\
    \midrule
     squi & NMI & 9.59\tiny{$\pm$0.21} & 9.59\tiny{$\pm$0.27} & 9.56\tiny{$\pm$0.19} & \textbf{9.67}\tiny{$\pm$0.13}\\
    & F1 & \textbf{55.43}\tiny{$\pm$0.03} & 53.74\tiny{$\pm$3.77} & 36.42\tiny{$\pm$1.85} & 36.51\tiny{$\pm$2.39} \\
    \midrule
     rom & NMI & 33.17\tiny{$\pm$1.25} & \textbf{40.05}\tiny{$\pm$0.82} & 35.99\tiny{$\pm$0.95} & 36.99\tiny{$\pm$0.61}\\
     & F1 & 33.57\tiny{$\pm$2.15} & 35.16\tiny{$\pm$1.34} & 35.30\tiny{$\pm$0.97} & \textbf{35.43}\tiny{$\pm$1.07}\\
    \midrule
     mine & NMI & \textbf{0.08}\tiny{$\pm$0.02} & 0.07\tiny{$\pm$0.02} & 0.04\tiny{$\pm$0.01} & 0.06\tiny{$\pm$0.00}\\
     & F1 & 76.15\tiny{$\pm$2.25} & 76.05\tiny{$\pm$2.16} & 73.77\tiny{$\pm$3.40} & \textbf{76.48}\tiny{$\pm$1.56}\\
    \midrule
     tol & NMI & 6.33\tiny{$\pm$0.94} & 6.18\tiny{$\pm$0.67} & 4.42\tiny{$\pm$0.54} & \textbf{6.73}\tiny{$\pm$0.59}\\
     & F1 & 73.89\tiny{$\pm$4.00} & 68.60\tiny{$\pm$5.95} & 68.45\tiny{$\pm$7.57} & \textbf{73.66}\tiny{$\pm$2.10} \\
    \bottomrule
    \end{tabular}
    \end{scriptsize}
    \label{tab:full_ablation}
\end{table}

\begin{figure}
    \centering    \includegraphics[width=0.85\textwidth]{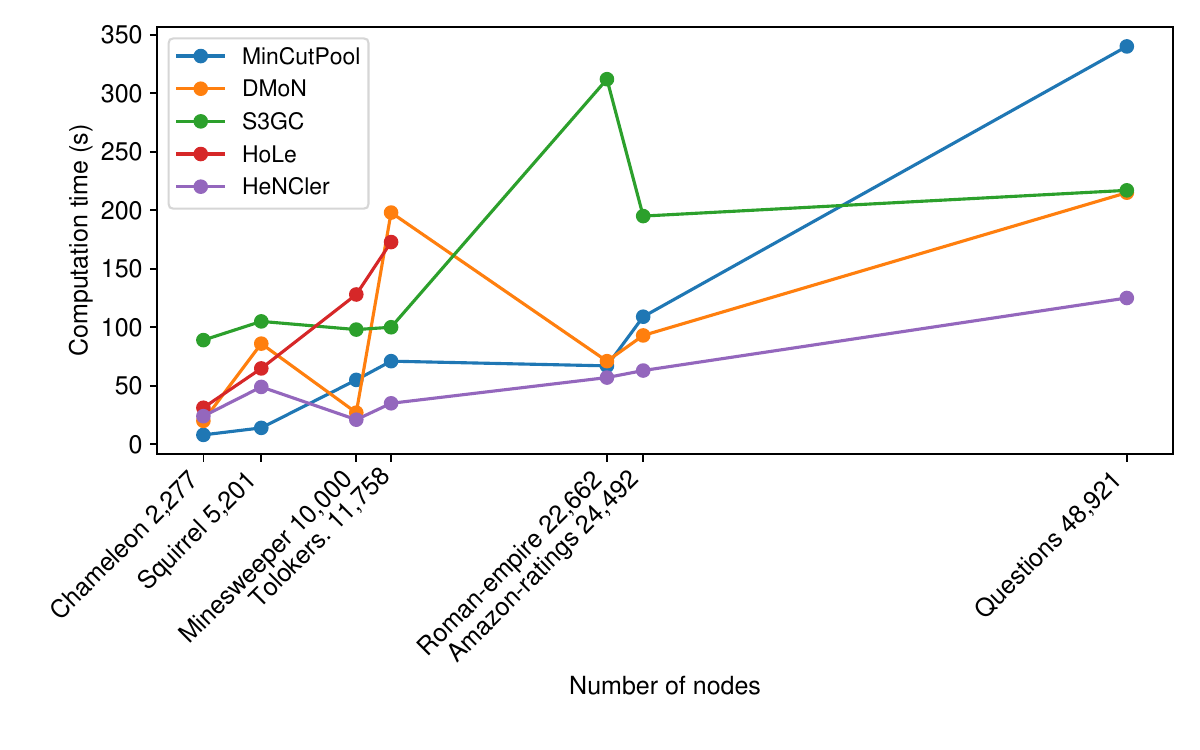}
    \caption{Fig. 1: Computation times of MinCutPool, DMoN, S$^3$GC, HoLe, and HeNCler with respect to the number of nodes in the datasets. HeNCler exhibits linear scalability and is not sensitive to the number of edges, unlike DMoN, which shows a significant spike in computation time for the Tolokers dataset due to its high edge count. For HoLe, computation times are reported only up to Tolokers, as later datasets caused out-of-memory errors.}
    \label{fig:compute}
\end{figure}
\begin{figure}[ht]
\begin{center}
\includegraphics[width=0.7\textwidth]{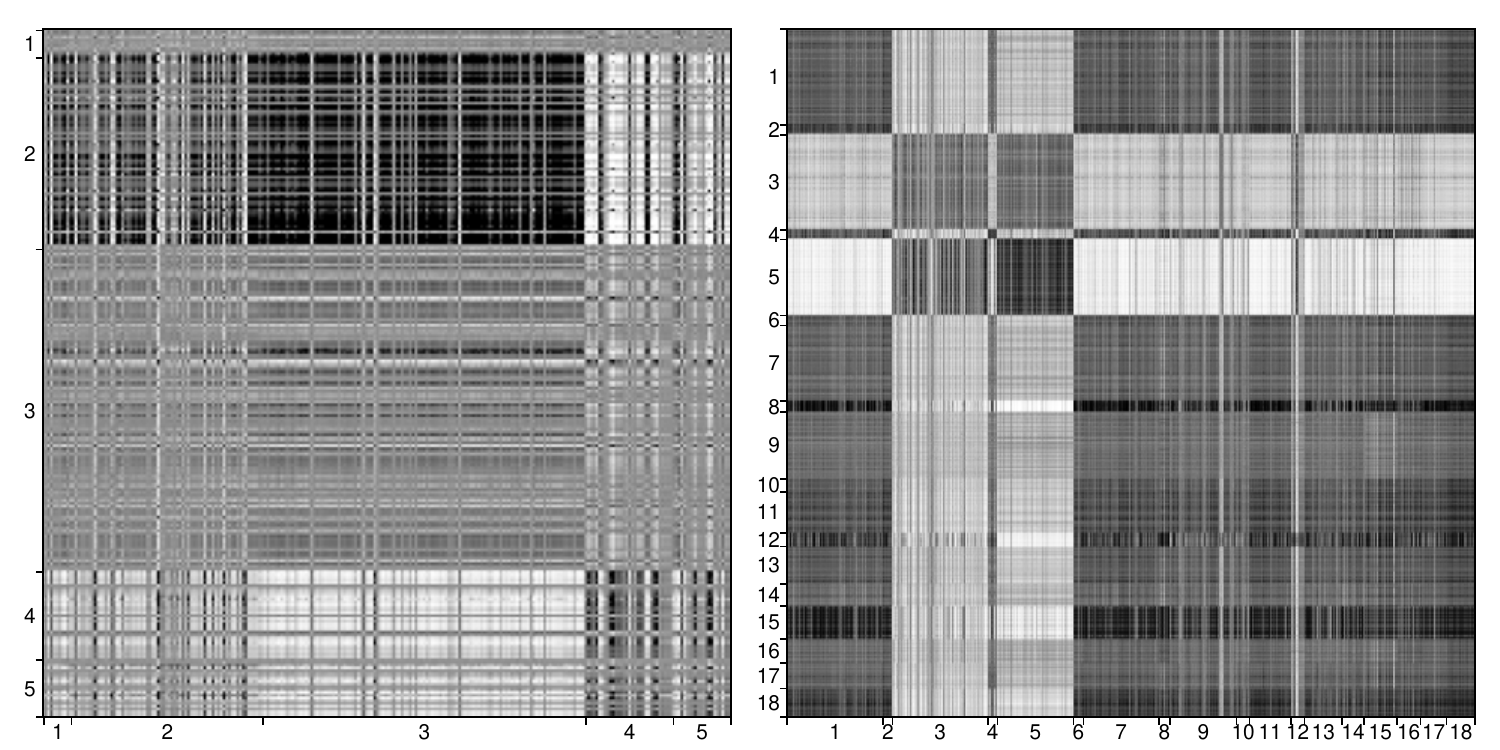} 
  \caption{The learned matrix $\bm{S}=\bm{\Phi\Psi}^\top$ for the Wisconsin (left) and Roman-empire (right) dataset. Rows and columns are grouped according to ground-truth node labels.} \label{fig:roman_heatmap}
\end{center}
\end{figure}

\section{Computational complexity}\label{app:compute}
The space and time complexity of the current implementation of HeNCler are both linear w.r.t. the number of nodes $\mathcal{O}(|\mathcal{V}|)$. Whereas MinCutPool and DMoN need all the node attributes in memory to calculate the loss w.r.t. the full adjacency matrix, HeNCler is easily adaptable to work with minibatches which reduces space complexity to the minibatch size $\mathcal{O}(|\mathcal{B}|)$. Although HeNCler relies on edge reconstruction, the edge sampling avoids quadratic complexity w.r.t. the number of nodes, and is specifically designed to scale with the number of nodes, rather than the number of edges. Assuming the graphs are sparse, we add an overview of space and time complexity w.r.t. the number of nodes and edges in Table \ref{tab:qualitative comparison}. We do not include HoLe in this comparison, as no theoretical complexities have been reported in the paper.

\begin{table*}[!t]
\caption{Qualitative comparison of HeNCler with several baselines. In the table, $|\mathcal{V}|$, $|\mathcal{B}|$, and $|\mathcal{E}|$ denote the total number of nodes, the mini-batch size, and the number of edges respectively.}
\begin{center}
\begin{small}
\begin{sc}
\begin{tabular}{lccccc}
\toprule
 & \multicolumn{4}{c}{baselines} & ours \\ 
\cmidrule(l){2-5} \cmidrule(l){6-6}
 & MinCutP. & DMoN & S$^3$GC & MUSE & HeNCler \\
\midrule
\midrule
Space complexity & $\mathcal{O}(|\mathcal{V}|^2)$ & $\mathcal{O}(|\mathcal{V}|+|\mathcal{E}|)$ & $\mathcal{O}(|\mathcal{B}|)$ & $\mathcal{O}(|\mathcal{V}|+|\mathcal{E}|)$  & $\mathcal{O}(|\mathcal{B}|)$\\
Time complexity & $\mathcal{O}(|\mathcal{V}|+|\mathcal{E}|)$ & $\mathcal{O}(|\mathcal{V}|+|\mathcal{E}|)$ & $\mathcal{O}(|\mathcal{V}|)$ & $\mathcal{O}(|\mathcal{V}|+|\mathcal{E}|)$  & $\mathcal{O}(|\mathcal{V}|)$\\
\bottomrule
\end{tabular}
\end{sc}
\end{small}
\end{center}
\label{tab:qualitative comparison}
\end{table*} 

We train MinCutPool, DMoN, and HeNCler for 300 iterations; and S$^3$GC for 30 iterations. For HoLe, we use 5 cluster-aware sparsification updates.
Experiments are conducted using a 16GB GPU and we report the computation times in Table \ref{tab:computationtimes}.
Figure \ref{fig:compute} visualizes these results w.r.t. the number of nodes in the graph, showing the linear time complexity of HeNCler and that it is insensitive to the number of edges. We conclude that HeNCler demonstrates fast computation times.

\begin{table*}[ht]
\caption{Computation times in seconds. OOM indicates an out-of-memory error on a 16GB GPU.}
\begin{center}
\begin{sc}
\begin{tabular}{lccccc}
\toprule
\multirow{2}*{Dataset} &\multicolumn{4}{c}{baselines} & ours \\ 
\cmidrule(l){2-5} \cmidrule(l){6-6}
 & MinCutP. & DMoN & S$^3$GC & HoLe & HeNCler \\
\midrule
\midrule
Chameleon       &   8 &  20 &  89 & 31 & 24 \\
Squirrel        &  14 &  86 & 105 & 65 & 49 \\
Roman-empire    &  67 &  71 & 312 & OOM & 57 \\
Amazon-rating   & 109 &  93 & 195 & OOM & 63 \\
Minesweeper     &  55 &  27 &  98 & 128 & 21 \\
Tolokers        &  71 & 198 & 100 & 173 & 35 \\
Questions       & 340 & 215 & 217 & OOM & 125  \\
\bottomrule
\end{tabular}
\end{sc}
\end{center}
\label{tab:computationtimes}
\end{table*} 
\section{Visualization of learned similarities}\label{sec:disc}
We visualize the learned similarity matrix $\bm{S}=\bm{\Phi \Psi^\top}$ for two datasets in Figure \ref{fig:roman_heatmap}.
These matrices are generally asymmetric, with the asymmetry particularly pronounced in the directed graph of the Wisconsin dataset. In contrast, the Roman-Empire dataset, which is represented by an undirected graph, exhibits less asymmetry in the learned similarity matrix. This demonstrates the adaptability of HeNCler to handle both directed and undirected graphs. Further, given the observable block structures, the learned similarities are meaningful w.r.t. to the ground truth node labels. 
Note however that our model operates in the primal setting and directly projects the learned mappings $\bm{\phi}$ and $\bm{\psi}$ to their final embeddings $\bm{e}$ and $\bm{r}$ using $\bm{U}$ and $\bm{V}$ respectively, avoiding quadratic space complexity and cubic time complexity of the SVD. This is the motivation of employing a kernel based method, and exploiting the primal-dual framework that comes with it. In fact, the matrices in Figure \ref{fig:roman_heatmap} are only constructed for the sake of this visualization.

\end{document}